\documentclass[conference]{IEEEtran}
\usepackage{times}

\usepackage[numbers]{natbib}
\usepackage{multicol}
\usepackage[bookmarks=true]{hyperref}

\let\labelindent\relax

\usepackage[numbers]{natbib}
\usepackage{multicol}
\usepackage[bookmarks=true]{hyperref}
\usepackage{amsmath}
\usepackage{amssymb}
\usepackage{color}
\usepackage{graphicx}
\usepackage{appendix}
\graphicspath{{images/}}
\usepackage{amsmath}
\usepackage{cleveref}
\usepackage{xspace}
\usepackage{subcaption}
\usepackage[normalem]{ulem}

\usepackage{setspace}

\usepackage{amsmath}
\usepackage{amsfonts}
\usepackage{amssymb}
\usepackage{amsthm}
\usepackage{bm}
\usepackage{bbm}
\usepackage{mathtools}
\usepackage{enumitem}
\usepackage{thmtools,thm-restate}
\usepackage{algorithm}
\usepackage{algorithmicx}
\usepackage{algpseudocode}
\usepackage{color}
\usepackage{graphicx}
\usepackage{comment}

\def\Rbb{\mathbb{R}}

\def\R{\Rbb}

\def\*{\star}

\newcommand{\energize}{\mathrm{energize}}

\usepackage[dvipsnames]{xcolor}

\newcommand{\mxi}{\boldsymbol{\xi}}

\newcommand{\mtau}{\boldsymbol{\tau}}

\newcommand{\Lag}{\mathcal{L}}
\newcommand{\Ham}{\mathcal{H}}

\newcommand{\q}{\mathbf{q}}
\newcommand{\qd}{{\dot{\q}}}
\newcommand{\qdd}{{\ddot{\q}}}

\newcommand{\vv}{\mathbf{v}}

\newcommand{\f}{\mathbf{f}}
\newcommand{\h}{\mathbf{h}}

\newcommand{\zero}{\mathbf{0}}

\newcommand{\A}{\mathbf{A}}
\newcommand{\B}{\mathbf{B}}

\newcommand{\I}{\mathbf{I}}

\newcommand{\M}{\mathbf{M}}

\newcommand{\mP}{\mathbf{P}}

\newcommand{\V}{\mathbf{V}}

\newcommand{\wt}[1]{{\widetilde{#1}}}
\newcommand{\wh}[1]{{\widehat{#1}}}

\usepackage{amsmath}  
\usepackage{amssymb}
\usepackage{accents}  
\usepackage{amsthm}

\theoremstyle{plain}
\newtheorem{theorem}{Theorem}[section]
\newtheorem{lemma}[theorem]{Lemma}
\newtheorem{proposition}[theorem]{Proposition}
\newtheorem{corollary}[theorem]{Corollary}

\theoremstyle{definition}
\newtheorem{definition}[theorem]{Definition}

\theoremstyle{remark}
\newtheorem{remark}[theorem]{Remark}

\makeatletter
\let\save@mathaccent\mathaccent
\newcommand*\if@single[3]{%
  \setbox0\hbox{${\mathaccent"0362{#1}}^H$}%
  \setbox2\hbox{${\mathaccent"0362{\kern0pt#1}}^H$}%
  \ifdim\ht0=\ht2 #3\else #2\fi
  }
\newcommand*\rel@kern[1]{\kern#1\dimexpr\macc@kerna}
\newcommand*\widebar[1]{\@ifnextchar^{{\wide@bar{#1}{0}}}{\wide@bar{#1}{1}}}
\newcommand*\wide@bar[2]{\if@single{#1}{\wide@bar@{#1}{#2}{1}}{\wide@bar@{#1}{#2}{2}}}
\newcommand*\wide@bar@[3]{%
  \begingroup
  \def\mathaccent##1##2{%
    \let\mathaccent\save@mathaccent
    \if#32 \let\macc@nucleus\first@char \fi
    \setbox\z@\hbox{$\macc@style{\macc@nucleus}_{}$}%
    \setbox\tw@\hbox{$\macc@style{\macc@nucleus}{}_{}$}%
    \dimen@\wd\tw@
    \advance\dimen@-\wd\z@
    \divide\dimen@ 3
    \@tempdima\wd\tw@
    \advance\@tempdima-\scriptspace
    \divide\@tempdima 10
    \advance\dimen@-\@tempdima
    \ifdim\dimen@>\z@ \dimen@0pt\fi
    \rel@kern{0.6}\kern-\dimen@
    \if#31
      \overline{\rel@kern{-0.6}\kern\dimen@\macc@nucleus\rel@kern{0.4}\kern\dimen@}%
      \advance\dimen@0.4\dimexpr\macc@kerna
      \let\final@kern#2%
      \ifdim\dimen@<\z@ \let\final@kern1\fi
      \if\final@kern1 \kern-\dimen@\fi
    \else
      \overline{\rel@kern{-0.6}\kern\dimen@#1}%
    \fi
  }%
  \macc@depth\@ne
  \let\math@bgroup\@empty \let\math@egroup\macc@set@skewchar
  \mathsurround\z@ \frozen@everymath{\mathgroup\macc@group\relax}%
  \macc@set@skewchar\relax
  \let\mathaccentV\macc@nested@a
  \if#31
    \macc@nested@a\relax111{#1}%
  \else
    \def\gobble@till@marker##1\endmarker{}%
    \futurelet\first@char\gobble@till@marker#1\endmarker
    \ifcat\noexpand\first@char A\else
      \def\first@char{}%
    \fi
    \macc@nested@a\relax111{\first@char}%
  \fi
  \endgroup
}
\makeatother

\begin{document}

\title{Fabrics: A Foundationally Stable Medium for Encoding Prior Experience}
\author{Nathan Ratliff and Karl Van Wyk}

\maketitle

\begin{abstract}
Most physical systems have dynamics functions that are just a nuisance to policies. Torque policies, for instance, usually have to effectively invert the natural classical mechanical dynamics to get their job done. Because of this, we often use controllers to make things easier on policies. For instance, inverse dynamics controllers wipe out the physical dynamics so the policy starts from a clean slate. That makes learning easier, but still the policy needs to learn everything about the problem, including aspects of a solution which are common to many other problems, such as how to make the end-effector move in a straight line, how to avoid joints and self collisions, how to avoid obstacles, etc. Over the past few years it's become standard to formulate learning not in C-space, but in end-effector space and use controllers such as Operational Space Control (OSC) to capture some of these commonalities. These controllers, whether inverse dynamics or OSC, reshape the natural dynamics of the system into a different second-order dynamical system whose behavior is more useful. And the trend is, the more useful behavior we can pack into these reshaped systems, the easier it is to learn policies.

However, OSC is from the 80's, and captures only straight line end-effector motion. There's a lot more behavior we could and should be packing into these systems. Earlier work \cite{optimizationFabricsForBehavioralDesignArXiv2020,ratliff2021finsler,vanwyk2022geometricfabrics} developed a theory that generalized these ideas and constructed a broad and flexible class of second-order dynamical systems which was simultaneously expressive enough to capture substantial behavior (such as that listed above), and maintained the types of stability properties that make OSC and controllers like it a good foundation for policy design and learning. This paper, motivated by the empirical success of the types of fabrics used in \cite{xie2022ngf}, reformulates the theory of fabrics into a form that's more general and easier to apply to policy learning problems. We focus on the stability properties that make fabrics a good foundation for policy synthesis. Fabrics create a fundamentally stable medium within which a policy can operate; they influence the system's behavior without preventing it from achieving tasks within its constraints. When a fabrics is \emph{geometric} (path consistent) we can interpret the fabric as forming a \emph{road network} of paths that the system wants to follow at constant speed absent a forcing policy, giving geometric intuition to its role as a prior. The policy operating over the geometric fabric acts to modulate speed and steers the system from one road to the next as it accomplishes its task.

We reformulate the theory of fabrics here rigorously and develop theoretical results characterizing system behavior and illuminating how to design these systems, while also emphasizing intuition throughout.
\end{abstract}

\IEEEpeerreviewmaketitle

\section{Introduction}
\label{sec:intro}

Policies all operate on underlying system dynamics: what the robot wants to do absent external control. These dynamics can be as straightforward as the underlying classical mechanical dynamics of the robot, where the system's inertia defines its Riemannian geometry, the network of paths the system would travel along absent gravity and frictions (see \cite{ratliff2021finsler} for a discussion of the connection between classical mechanics and Riemannian geometry). That mechanical system of paths, though, is often irrelevant to tasks and a hindrance to achieving a desired behavior. 

Therefore, control systems often work to reshape that geometry into something more relevant, or at least less disruptive. For instance, inverse dynamics control \cite{spong2006robot} removes the geometry entirely, replacing it with a Euclidean geometry in C-space (a blank slate), such that additional controllers can generate a desired behavior without competing with the native system geometry.

Operational space control \cite{khatib1987unified} builds upon this idea and not only clears the geometry, but also reshapes it into something more relevant to the task. Specifically, it replaces the physical geometry with a different Riemannian geometry where geodesics move the end-effector in straight lines. Policies then build off that more useful geometry to define useful task behavior. Since tasks are often more easily described in the end-effector space, this starting geometry is highly relevant to many problems---it encodes useful prior information.

However, operational space control captures only a small fraction of the commonalities among tasks. Most tasks, for instance, require some form of obstacle awareness, such as avoidance or attraction toward a surface (e.g. grasping). Moreover, robots generally avoid joint limits and self-collisions, and approach targets from a particular direction (e.g. approach a table orthogonal to a surface when touching it). Many of these behavioral elements can and should be factored out and encoded into the reshaping controller itself, and ideally we should extract these common behavioral components from data.

In this work, we formalize this concept of an underlying behavior shaping controller into what we call \emph{fabrics}. We define fabrics rigorously as \emph{conservative} autonomous second-order differential equations, show how to construct them by \emph{energizing} a generating system, and thoroughly characterize their intrinsic stability properties. This theory supports the type of fabric design used in recent fabric learning work such as \cite{xie2022ngf} and gives the theoretical foundations for encoding prior information into an underlying fabric and training policies over the top of it.

When the fabric has a particular geometric path consistency property, we call it a \emph{geometric} fabric. This path consistency gives the fabric a speed invariant road network of paths that guide the system around obstacles and other constraints and broadly encode important prior information. Policies navigating these fabrics generally follow the network of paths and need only choose when and how strongly to push the system from one path to another and how to regulate energy along the way. We provide a number of theoretical statements characterizing energy regulation and system convergence, including convergence to desired goals (the zero set of a forcing term). We tailor the theory to providing insight into the design and training of fabrics and the policies residing on them. 

Importantly, these shared fabrics, especially those learned from data \cite{xie2022ngf}, constitute well-informed priors on behavior. We discuss this perspective throughout this work. Generally, when we us the term \emph{prior}, we mean it in the broader sense than purely Bayesian probabilistic. We simply mean it acts as a way to inject prior experience into the system to improve the sample efficiency of training policies (including the manual design of policies which is, itself, an information theoretic learning process---an iterative process resulting in a policy expected to generalize to novel situations.) Note that fabrics can be used to develop probabilistic priors by running stochastic policies over them to generate distributions of trajectories. But the underlying fabric itself, which captures the essence of the encoded information in its geometry, is not probabilistic.

\subsection{Related work}

Since robots are physical systems, their dynamics are well-understood and governed by the Euler-Lagrange equation, as characterized in any number of introductory robotics text books \cite{craig2005introduction, spong2006robot}.
The mathematics of these systems is sophisticated, with Lagrangian symmetries giving rise to conserved quantities such as energy conservation \cite{taylor2005classical}, and control theorists have exploited those mathematical properties thoroughly for the stable design of complex nonlinear controllers by reshaping the systems into different, more favorable, dynamical systems \cite{behal2009lyapunov, bullo2019geometric}. These fundamental equations have also been used in the creation of modern robot simulators like \cite{macklin2019non, makoviychuk2021isaac}, which are critical tools for designing or learning robot control policies. 

In classical systems, the Euler-Lagrange equations decompose into two parts. One part is the closed system's inertial equations. This component can be shown to be geometric in nature in the sense that it produces speed-invariant paths through space \cite{bullo2019geometric,ratliff2021finsler}. The system, under the influence of only its inertia, follows the same path regardless of speed (or more generally accelerations along the direction of motion). These geometries are Riemannian, and the system's mass matrix is the Riemannian metric (see \cite{ratliff2021finsler} for a derivation). The second part includes additional forcing and damping terms. Both of these components are required to accurately model the complex, nonlinear physical phenomena in the real world. Reshaping these models into useful behavioral systems within the same class of classical mechanical systems (Riemannian geometries) is common \cite{bloch2000energyshapingI,bloch2001energyshapingII,bullo2019geometric,khatib1987unified}, and interestingly, Riemannian geodesics have also been shown to model large segments of human motion in \cite{klein2022riemannian, neilson2015riemannian}, indicating that energy efficient human motion can be largely captured by following geometric paths. However, these Riemannian systems are fundamentally limited in their expressivity for two reasons: their metrics can only be a function of position (no velocity), and the metric plays a double role of both defining the geometry of paths itself and specifying how one sub-system weights together with another. Broader classes of second-order differential equations, such as Riemannian Motion Policies (RMPs) for motion generation in \cite{ratliff2018riemannian}, aren't limited in these ways and have been shown empirically to have high-capacity for representing intricate behaviors, but are less well understood.

Recently, \cite{vanwyk2022geometricfabrics} generalized classical mechanical models to what are called geometric fabrics, building off a type of system termed a \emph{bent Finsler systems}, to expand the modeling capacity of these types of systems specifically to remove those above limitations. Geometric fabrics capture the flexibility of RMPs while being provably stable and maintaining a form of (non-Riemannian) geometric path consistency, building off the mathematics of Finsler and Spray geometry \cite{ratliff2021finsler}. With bent Finsler systems, behavioral designers were able to engineer policies that can outperform both the classical mechanical systems of geometric control and RMPs. These systems were also shown to outperform linear dynamical systems (such as Dynamic Movement Primitives, DMPs), RMPs, and a variety of baseline neural architectures like Long Short-Term Memory (LSTM) networks in learning contexts \cite{xie2022ngf}. These systems superficially resemble artificial potential fields \cite{khatib1986real}, but are built to enable the design of the \emph{geometry} rather than the potential function, which improves their regularity and dramatically boosts performance in practice. Behavior can be written directly into the underlying road network of paths, reshaping the system geometry, rather than relying on the potential function to push the system (fight) against a less relevant natural geometry.

Independently, Bylard et al \cite{bylard2021pbds} developed what are called Pullback Bundle Dynamical Systems (PBDS) as a rigorously covariant version of the Geometric Dynamical Systems (GDS) developed earlier in \cite{cheng2018rmpflow}. They approached the problem as developing Riemannian metrics on the tangent bundle (space of positions and velocities) of a given manifold. While rigorous, those systems are analogous in their representational capacity to the Lagrangian fabrics outlined in \cite{optimizationFabricsForBehavioralDesignArXiv2020} and from the analysis in \cite{vanwyk2022geometricfabrics}, it's now know they lack the flexibility to independently represent both the geometry of paths and the metric independent of one another. For that reason, PBDS rely heavily on non-geometric potential functions similar to the standard potential shaping techniques of geometric control \cite{bullo2019geometric}. The perspective we develop here builds from the results on nonlinear (spray) geometries of paths detailed in \cite{ratliff2021finsler}, and requires less technical machinery than developing metrics directly on the tangent bundle. Finsler fabrics and more broadly Lagrangian fabrics are analogous to PBDS, but geometric fabrics are a fundamentally more expressive class of systems.

All of these earlier works, while elegant in their generalization of classical mechanical or Riemannian systems, often require complex tensor calculations, such as evaluating the Euler-Lagrange equation, to fully follow the theory. This complexity introduces challenges, especially when learning is involved. Here, we reformulate these systems in a way that's more intuitive, easier to handle both conceptually and implementationally, and emphasizes the role they play as fundamentally stable mediums for guiding policies.

\subsection{A note on generalized notation}
\label{sec:GeneralizedNotation}

Often a classical mechanical, (bent) Finsler, RMP system, etc. takes the form $\M(\q,\qd)\qdd + \mxi(\q,\qd) = \zero$, and when forced by some potential function $\psi(\q)$, and damped by a dissipating term $\B(\q,\qd) \qd$, we get
\begin{align}
    \M \qdd + \mxi = -\partial\psi - \B\qd.
\end{align}
The resulting acceleration is
\begin{align} \label{eqn:DecomposedSystemPrelim}
    \qdd &= -\M^{-1}\mxi - \M^{-1}\big(\partial\psi + \B\qd\big) \\\label{eqn:DecomposedSystem}
    &= \wt{\h}(\q,\qd) + \f(\q,\qd),
\end{align}
where $\wt{\h} = -\M^{-1}\mxi$ and $\f=-\M^{-1}\big(\partial\psi + \B\qd\big)$. That first term $\wt{\h}$ is conservative and often geometric (and/or unbiased in the sense that it won't push the system away from rest) and the second term both forces away from $\wt{\h}$ and regulates the injection and dissipation of energy.

In this work, we address general decomposed systems of the form given in Equation~\ref{eqn:DecomposedSystem}. We use $\wt{\h}$ to denote a fabric (conservative term) and the $\f$ to denote a forcing term that pushes against the fabric and regulates energy. The tilde denotes that the term is conservative (a fabric), to distinguish it from a generator that creates a fabric through energization (see Definition~\ref{def:Energization} and Lemma~\ref{lma:EnergizingCreatesFabrics}).

This decomposition is very general and covers many systems, including the ones above. If $\wt{\h}$ has an associated system metric $\M$, it's often useful to think of forcing policies as force functions such as $\pi(\q,\qd) = -\partial\psi-\B\qd$ which are transformed by the system metric into the forcing term $\M^{-1}\pi(\q,\qd)$ such as in Equation~\ref{eqn:DecomposedSystemPrelim}. Note that strong Eigen-directions of $\M$ trim away components of the force. In that sense, $\M$ defines the system priorities, intuitively defining which directions in space important to $\wt{\h}$ and which directions that aren't.

\subsection{Overview}

We begin in Section~\ref{sec:FundamentalStability} with a series of results characterizing the fundamental stability of fabrics as a medium for policies to operate on. Throughout, we define and develop the theory of fabrics generally but also detail the important role path consistency plays in the more specific case of \emph{geometric} fabrics which form a concrete road network of paths for policies to operate across.

We start by defining fabrics to be conservative second-order autonomous differential equations in Definition~\ref{def:Fabrics} and show that energy conservation, by itself, gives the fabric important stability properties. Terminologically, in Definition~\ref{def:NavigationAcrossFabrics}, we decompose the full system into $\qdd = \wt{\h} + \f$, where $\wt{\h}$ is the fabric and $\f$ is the policy, and together they form a \emph{forced system}. This terminology derives from the force form $\M(\q,\qd) \qdd + \mxi = \pi(\q,\qd)$ where $\M$ is a symmetric positive definite system metric and we have the relations $\wt{\h} = -\M^{-1}\mxi$ and $\f = \M^{-1}\pi$. Here $\pi$ is called a \emph{force} policy. Intuitively, a system traveling along a fabric will always maintain constant energy if the policy does nothing, and the policy can always simply dissipate energy to come to a stop. Over a bounded period of time, a bounded policy can only inject a finite amount of energy into the system, so it always has the means to easily bring the system back to rest. Fabrics, in that sense, innately form a fundamentally stable medium across which the policy operates.

We show in Lemma~\ref{lma:EnergizingCreatesFabrics} that we can always transform any given second-order autonomous system into a fabric simply by speeding up or slowing down along the direction of motion, and Definition~\ref{def:Energization} gives a specific \emph{energization} transform that does that. Importantly, Proposition~\ref{prop:GeometricFabrics} then shows that if the underlying generating system is geometric (path consistent), energization stabilizes it without changing the collection of paths (since it operates entirely by accelerating along the direction of motion which is known to leave paths unchanged in geometric systems).

Then Proposition~\ref{prop:NavigatingFabrics} shows that any policy operating across a geometric fabric can be decomposed into a zero-work energy-preserving term which bends (or steers) the paths without changing the energy, and an energy regulation term which modulates speed along the direction of motion without changing the path. All policies thereby act to simply modulate the underlying fabric's energy while steering the system. When training policies, one can potentially exploit this observation to define data efficient policy parameterizations.

Section~\ref{sec:FundamentalStability} finishes with a discussion of convergence to the zero set of the forcing policy. Broadly, there are many cases where goals can be characterized by zero sets of some vector field. For instance, the local minima of a potential are the zero sets of its gradient. A forcing policy is a vector field that vanishes when it no longer wants to move the system, so the zero set of the forcing policy is a good characterization of the policy's goals.
Proposition~\ref{prop:ForcedSystemConvergence} presents some general conditions under which the forced system converges to the zero set. One of those conditions is the practical statement that if the system (with bounded accelerations) converges, it must converge to the zero set. That comes from the simple observation that the fabric is conservative and therefore wouldn't itself push the system from rest (zero energy). So if it comes to reset at the zero set, neither the fabric nor the policy wants it to move from there. It's often straightforward to design convergent systems that dissipate energy properly to bring the system to rest at a zero set, so even if we can't otherwise prove global convergence of the system, we can design practically convergent systems which are guaranteed to be at the policy's zero set when they converge. Moreover, these observations suggest that given a goal, we can parameterize the policy to ensure the policy is zero if and only if it's at the goal. Then a training system needs only learn how to modulate energy effectively to converge nicely to that zero set.

Section~\ref{sec:EnergyRegulation} moves into a more complete discussion of theoretical conditions on energy regulation. Propositions~\ref{prop:EnergyCapping1} and ~\ref{prop:EnergyCapping2} give some policy parameterizations for which we can guarantee bounded energy and a natural form of energy regulation. The main result of this section is Theorem~\ref{thm:main}, which gives a specific energy regulation formula under which any forced fabric can be guaranteed to converge to the zero set of the policy provided there exists what we call a compatible potential which we use to guide the energy regulation. 

Section~\ref{sec:ForcingEnergizedFabrics} gives a final stability analysis for a common case where the fabric has a corresponding system metric and is being forced by a damped potential function. This setting is similar to the geometric fabric setting of \cite{vanwyk2022geometricfabrics}, but more general. Importantly, we allow the underlying geometric fabric to be arbitrary and paired with any system metric. It's typically much easier to design and implement such systems than the bent Finsler geometries described in \cite{vanwyk2022geometricfabrics}.

Finally, we summarize the takeaways in Section~\ref{sec:Conclusions}.
\section{The Fundamental Stability of Fabrics}
\label{sec:FundamentalStability}

Fabrics are stable autonomous second-order differential equations that can form well-informed priors on policies by encoding behavioral information common across many tasks. Individual control policies use fabrics by navigating across them. In this section, we define fabrics and characterize their utility and fundamental stability.

Throughout this work we use the multivariate calculus notational conventions outlined in \cite{optimizationFabricsForBehavioralDesignArXiv2020}. Note that in earlier work we built in specific conditions to handle boundary conformance for manifolds with a boundary. Those boundary conditions often require systems to be unbounded (e.g. accelerations or metrics approach infinity), which is impractical for real-world implementation and numerical integration. More recently, \cite{vanwyk2022geometricfabrics} described how to integrate explicit hard constraints into the definition of systems like the ones we consider here; constraint forces effectively fold into the forcing policy making them conceptually simpler. In many cases practical implementations use terms that are softer and better conditioned to smoothly avoid constraints (e.g. added potential function in the framework of \cite{vanwyk2022geometricfabrics}). We, therefore, cover only the unconstrained setting here, and refer the reader to \cite{vanwyk2022geometricfabrics} for details on how to incorporate hard constraints.

\begin{definition}[Fabrics]\label{def:Fabrics}
An autonomous differential equation $\qdd = \wt{\h}(\q,\qd)$ is a \emph{fabric} if it conserves a Finsler energy $\Lag(\q,\qd)$.
\end{definition}

This definition states that a fabrics is simply a conservative second-order autonomous differential equation. That conservation property is what makes the fabric a nice stable medium for policy design. The following Lemma shows that the fabric itself doesn't attempt to push a system from rest. This property will enable policies to reliably navigate the fabric and converge to any given desired goal.

\begin{lemma} \label{lma:FabricsAreUnbiased}
If $\wt{\h}$ is a fabric, then $\wt{\h}(\q,\zero) = \zero$.
\end{lemma}
\begin{proof}
Let $\Lag$ be the fabric's conserved Finsler energy. Finsler energies can be written $\Lag = \frac{1}{2}\qd^T\M_\Lag\qd$ where $\M_\Lag = \partial^2_{\qd\qd} \Lag$ (see \cite{ratliff2021finsler}), so $\qd=\zero$ if and only if $\Lag = 0$. If $\wt{\h}(\q,\zero)\neq \zero$ at time $t$, by continuity, 
there exists an $\epsilon>0$ such that $\qd\neq\zero$ at time $t+\epsilon$. But that would mean the energy changes which contradicts the fabrics conservation property. Therefore, $\wt{\h}(\q,\zero) = \zero$.
\end{proof}

\begin{remark}
Lemma~\ref{lma:FabricsAreUnbiased} shows that fabrics as defined in Definition~\ref{def:Fabrics} are \emph{unbiased} in the sense that they can influence the system's behavior while in motion, but vanish when the system stops. In other words, a system at rest remains at rest, allowing convergence regions to be entirely governed by the zero sets of other forcing terms (see Definition III.6 in \cite{optimizationFabricsForBehavioralDesignArXiv2020} for a precise description.)
\end{remark}

\begin{definition}[Navigating across fabrics] \label{def:NavigationAcrossFabrics}
Let $\f(\q,\qd)$ be a finite second-order differential equation. $\f$ is called a \emph{navigation policy}
when added to a fabric to form the system
\begin{align} \label{eqn:NavigatingSystem}
    \qdd = \wt{\h}(\q,\qd) + \f(\q,\qd).
\end{align}
We often say $\f$ \emph{navigates} across $\wt{\h}$. When the context is clear, we often refer to it simply as the \emph{policy}.
\end{definition}

We often describe the system in Equation~\ref{eqn:NavigatingSystem} as a \emph{forced} system because of it's relation to forcing policies as defined next.

\begin{definition}[Forcing policies] In many cases, there is a relevant positive-definite system metric $\M(\q,\qd)$ that can be used to shape navigating term (see Equation~\ref{eqn:DecomposedSystemPrelim} for the intuition). In that case, we usually write the system in its force form 
\begin{align}
    \M \qdd + \mxi = \mtau
\end{align}
where $\mxi = -\M\wt{\h}$ and $\mtau$ is an external force. The navigation term is then constructed using a \emph{forcing policy} denoted $\mtau = \pi(\q,\qd)$, matching standard policy notation. Since the metric $\M$ is invertible, there is a one-to-one correspondence between forcing policy and navigation term with $\f = \M^{-1}\mtau$. Again, when the context is clear, we often refer to it simply as the \emph{policy}.
\end{definition}

The following lemma collects together some previously proven results that characterize the energy conservation properties of fabrics.
 
\begin{lemma}[Properties of fabric energies] \label{lma:PropertiesOfFabricEnergies}
Let $\wt{\h}(\q,\qd)$ be a fabric with Finsler energy $\Lag$. 
The Hamiltonian has the property $\Ham_\Lag = \qd^T\partial_\qd\Lag - \Lag = \Lag$ and its time derivative takes the form $\dot{\Ham}_\Lag = \qd^T\big(\M_\Lag\qdd + \mxi_\Lag\big) = 0$ where $\M_\Lag\qdd + \mxi_\Lag = \zero$ are the Euler-Lagrange equations of $\Lag$ with $\M_\Lag = \partial^2_{\qd\qd}\Lag$ and $\mxi_\Lag = \partial_{\qd\q}\Lag\qd - \partial_\q\Lag$. The fabric conserves $\Lag$ so it has the property $\dot{\Ham}_\Lag = \qd^T\big(\M_\Lag\wt{\h} + \mxi_\Lag\big) = 0$. 
\end{lemma}
\begin{proof}
These results are proven in \cite{vanwyk2022geometricfabrics}, with the final fabric property following from conservation of energy.
\end{proof}

We use these properties to prove the following theorem which shows that fabrics are fundamentally stable in the sense that the energy of a forced system is bounded at any given time and can always be dissipated to bring the system to rest.

\begin{theorem}[Fundamental stability of fabrics] \label{thm:FundamentalStabilityOfFabrics}
Let $\wt{\h}(\q,\qd)$ be a fabric with Finsler energy $\Lag$. 
If $\f$ is a finite navigation policy, the corresponding forced system $\qdd=\wt{\h}+\f$ has finite energy after a finite time and will come to rest if the navigating term is set to $\f = \f_\mathrm{damp} = -\M_\Lag^{-1}\B(\q,\qd)\qd$, where $\B(\q,\qd)$ is any positive-definite damping matrix.
\end{theorem}
\begin{proof}
By Lemma~\ref{lma:PropertiesOfFabricEnergies} $\dot{\Lag}[\qdd] = \qd^T\big(\M_\Lag \qdd + \mxi_\Lag\big)$, so for our system we have
\begin{align}
    \dot{\Lag}[\wt{\h} + \f]
    &= \qd^T\big(\M_\Lag (\wt{\h} + \f) + \mxi_\Lag\big) \\
    &= \qd^T\big(\M_\Lag \wt{\h} + \mxi_\Lag\big) + \qd^T\M_\Lag\f \\
    &= \dot{\Lag}[\wt{\h}] + \qd^T\M_\Lag\f \\
    &= \qd^T\M_\Lag\f
\end{align}
since $\dot{\Lag}[\wt{\h}] = 0$. This is the work done by $\f$ on the system. The total work gives the energy after $T$ seconds as
\begin{align}
    \Lag(\q_T,\qd_T) = \Lag(\q_0,\qd_0) + \int_0^T \qd^T\M_\Lag\f dt,
\end{align}
which is finite. 

Choosing $\f = \f_\mathrm{damp} = -\M_\Lag^{-1}\B(\q,\qd)\qd$ after $T$ seconds gives energy change
\begin{align}
    \dot{\Lag}
    &= \qd^T\M_\Lag\big(-\M_\Lag^{-1}\B\qd\big) \\
    &= -\qd^T\B\qd < 0.
\end{align}
Since $\Lag$ is lower bounded, $\dot{\Lag} = -\qd^T\B\qd \to 0$ which means both $\qd\to\zero$ and $\qdd\to\zero$ as $t \to \infty$.
\end{proof}

\begin{remark}
See Corollary~\ref{cor:GeometricFabricStability} for a simplified form for the damper $\f_\mathrm{damp} = -\beta(\q,\qd)\qd$ where $\beta > 0$ is a scalar, which is appealing for preserving the path consistency of geometric fabrics as defined in Proposition~\ref{prop:GeometricFabrics}.
\end{remark}

Given any (finite) autonomous second-order differential equation $\h$, we can always accelerate along the direction of motion strategically to ensure any given Finsler energy is conserved. The following definition characterizes how to do that.

\begin{definition}[Energization] \label{def:Energization}
Let $\qdd = \h(\q,\qd)$ be a finite autonomous second-order differential equation, and let $\Lag$ be an energy. The \emph{energized} system is the transformed system defined as
\begin{align}
    &\qdd = \energize_\Lag\big[\h(\q,\qd)\big] = \h + \alpha \qd \\
    &\ \ \ \ \ 
    \mbox{where}\ \ \alpha 
        = -\frac{\qd^T\big(
                    \M_\Lag\h + \mxi_\Lag
                \big)
            }{\qd^T\M_\Lag\qd}
\end{align}
\end{definition}

This system transformation, which we call \emph{energization}, turns any $\h$ into a fabric by making it conservative. Note that the energy can be any Lagrangian, although it's common for that energy to be more specifically a Finsler energy.

\begin{lemma} \label{lma:EnergizingCreatesFabrics}
Let $\qdd = \h(\q,\qd)$ be a finite autonomous second-order differential equation, and let $\Lag$ be a Finsler energy. The energized system $\qdd = \energize_\Lag\big[\h\big]$ conserves $\Lag$ and is therefore a fabric. We call $\h$ the \emph{generator} of a fabric constructed in this way.
\end{lemma}
\begin{proof}
We show that the energized system conserves $\Lag$. By Lemma~\ref{lma:PropertiesOfFabricEnergies} $\dot{\Lag}[\qdd] = \qd^T\big(\M_\Lag \qdd + \mxi_\Lag\big)$, so after energization the time rate of change of $\Lag$ is
\begin{align*}
    \dot{\Lag}[\h + \alpha \qd]
    &= \qd^T\left[
            \M_\Lag \left(
                \h - \frac{
                    \qd^T\big(\M_\Lag\h + \mxi_\Lag\big)
                }{\qd^T\M_\Lag\qd} \qd
            \right) + \mxi_\Lag
        \right] \\
    &= \qd^T\M_\Lag \h 
        - \left(\frac{
                \qd^T\M_\Lag\h + \qd^T\mxi_\Lag
            }{\qd^T\M_\Lag\qd}\right) \qd^T\M_\Lag \qd \\
    &\ \ \ \ \ \ \ \ + \qd^T\mxi_\Lag \\
    &= \qd^T\M_\Lag \h - \qd^T\M_\Lag\h 
        - \qd^T\mxi_\Lag + \qd^T\mxi_\Lag \\
    &= 0.
\end{align*}
\end{proof}

In general, energization may change the behavior of a system since the path traced by a system often changes when the system speeds up or slows down. (E.g. an orbiting satellite will fall to earth if it slows and shoot out to space if it speeds up.) The following proposition characterizes the class of systems whose behavior is unaffected by energization.

\begin{definition}
A system $\qdd = \h(\q,\qd)$ is Homogeneous of Degree 2 (HD2) if $\h(\q,\alpha\qd) = \alpha^2 \h(\q,\qd)$ for $\alpha \geq 0$.
\end{definition}

\begin{remark}
An HD2 system modulates its accelerations in just the right way to maintain its path, independent of speed. If the system were constrained to follow a given path, speeding up by a factor of $\alpha$ would induce accelerations $\alpha^2$ times higher to maintain the path. An HD2 system has this scaling property built in to make its integral curves trace speed invariant paths. This speed invariance is a defining property of geometries \cite{ratliff2021finsler}.
\end{remark}

The next proposition characterizes the class of path consistent fabrics constructed by HD2 generators.

\begin{proposition}[Geometric Fabrics] \label{prop:GeometricFabrics}
Let $\qdd = \h(\q,\qd)$ be an HD2 generator, and let $\Lag(\q,\qd)$ be a Finsler energy. The paths traced by the fabric $\qdd = \wt{\h}(\q,\qd) = \energize_\Lag\big[\h\big]$ match those of $\h$. Moreover, the energized system is also HD2 so $\qdd = \energize_\Lag\big[\h\big] + \gamma \qd$ trace the same paths as its HD2 generator $\h$ for any time varying $\gamma$. Fabrics constructed this way are called \emph{geometric fabrics} and the class of geometric fabrics is the unique class of path consistent fabrics.
\end{proposition}
\begin{proof}
A property of HD2 systems is that they can accelerate along the direction of motion arbitrarily without changing the system's path \cite{ratliff2021finsler}. The energization transformation is defined as $\energize_\Lag\big[\h\big] = \h + \alpha \qd$ for a particular choice of $\alpha$. Therefore, the energized system is path consistent. Similarly, adding another term $\gamma\qd$ is also an acceleration along the direction of motion, so the paths remain consistent.

Examining the system under the specific energization coefficient, we see
\begin{align}
    \energize_\Lag\big[\h\big]
    &= \h - \frac{\qd^T\big(\M_\Lag \h + \mxi_\Lag\big)}{\qd^T\M_\Lag \qd}\qd \\
    &= \h + \A \big(\M_\Lag \h + \mxi_\Lag\big),
\end{align}
where $\A = -\qd \qd^T / \qd^T\M_\Lag \qd$. $\M_\Lag$ is HD0 (independent of velocity norm), and $\mxi_\Lag$ is HD2 (see \cite{ratliff2021finsler} for a discussion of these properties). $\A$ is also HD0 since $\M_\Lag$ is and there are two factors of $\qd$ in both the numerator and denominator. Therefore, the energized system is HD2 since $\h$ is HD2.

We prove uniqueness by contradiction. Suppose $\wt{\h}$ is a geometric fabrics but is not HD2. (If it is HD2, then it can be constructed as described above.) Then there exists a $(\q,\qd)$ where $\qdd = \wt{\h}(\q, \lambda \qd) \neq \lambda^2 \wt{\h}(\q,\qd)$ for some $\lambda\geq 0$. That means for that state and that $\lambda$, the integral curve starting at $(\q,\lambda\qd)$ will deviate from the integral curve starting at $(\q,\qd)$ after some finite time. Therefore, it can't be geometric which is a contradiction since it's a geometric fabric.
\end{proof}

\begin{remark}
The bent Finsler systems described in \cite{vanwyk2022geometricfabrics}, which can be characterized as generalizations of classical mechanical systems, are geometric fabrics as defined in Proposition~\ref{prop:GeometricFabrics}. The definition here, though, is broader and easier to work with in practice than the earlier definition. In bent Finsler systems, metrics must be defined by Finsler energies, requiring the application of Euler-Lagrange equations which can be computationally complex and challenging to implement. Under our definition here, allows metrics to be arbitrary HD0 positive semi-definite matrices dramatically simplifying design. The Finsler energy is still used for energization of the HD2 geometry generator, but it can remain simply since it needs only define the desired measure of speed, not the metrics. This simpler setup was already used in \cite{xie2022ngf} and is especially helpful where automatic differentiation is involved.
\end{remark}

\begin{corollary} \label{cor:GeometricFabricStability}
A forced geometric fabric of the form $\qdd = \wt{\h} = \energize_\Lag\big[\h\big] - \beta(\q,\qd)\qd$ with positive real valued $\beta > 0$ asymptotically comes to rest without deviating from the paths of the underlying HD2 system $\h$.
\end{corollary}
\begin{proof}
By Theorem~\ref{thm:FundamentalStabilityOfFabrics} the fabric will come to rest with $\f = -\M_\Lag^{-1}\B\qd$. Choose $\B = \beta(\q,\qd)\M_\Lag$ with $\beta > 0$. $\B$ is positive definite and $\f = -\M_\Lag^{-1}\beta(\q,\qd)\M_\Lag\qd = \beta \qd$. 
$\h$ is HD2, so $\wt{\h}$ is HD2 and follows the same paths as $\h$. Since $\f$ only accelerates along the direction of motion, by Proposition~\ref{prop:GeometricFabrics} the forced system $\qdd = \wt{\h} - \beta \qd$ maintains the same paths as $\wt{\h}$ and hence $\h$.
\end{proof}

We can think of geometric fabrics as forming a road network of paths through space. Without a navigation policy the system simply follows the nominal paths of the underlying fabric. The navigation policy then operates over the top of that nominal behavior, pushing the system from its current path to neighboring paths as needed. Similar to long-distant travel, when traveling to a distant goal, if the network of roads is well-designed, the navigation policy needs only set the system onto the right road up front, potentially do some minor switching of roads en route, and then, once close, pull the system off the major road networks to converge locally to the goal. Well-designed geometric fabrics can therefore significantly simplify the navigation policy. In that sense, they constitute a well-informed prior on behavior. 

The following Proposition shows that with geometric fabrics we can always view a navigation policy as a combination of zero-work steering (where the path bends but the energy remains constant) and speeding up or slowing down along the direction of motion (more precisely, path invariant energy regulation). 

\begin{proposition} \label{prop:NavigatingFabrics}
Let $\qdd = \energize_\Lag\big[\h\big]$ be a geometric fabric and let $\f$ be a navigation policy. Then the forced equation $\qdd = \energize_\Lag\big[\h\big] + \f$ can be written
\begin{align} \label{eqn:SteeredFabric}
    \qdd &= \energize_\Lag\big[\h+\f\big] + \gamma \qd
\end{align}
for $\gamma = \qd^T\M_\Lag\f/\qd^T\M_\Lag\qd \in \R$. The first term is a fabric which we call the \emph{steered fabric}, and the second term is an \emph{energy regulator}.
We can also write this system as
\begin{align} \label{eqn:RegulatedSteeringAcrossAFabric}
    \qdd = \Big(\energize_\Lag\big[\h\big] + \gamma \qd\Big) 
        + \mP_\perp\f
\end{align}
where $\mP_\perp = \M_\Lag^{-\frac{1}{2}}\big[\I-\wh{\vv}\wh{\vv}^T\big]\M_\Lag^{\frac{1}{2}}$ with $\vv = \M_\Lag^{\frac{1}{2}} \qd$ and $\wh{\vv} = \vv / \|\vv\|$ is a projection matrix. 
\end{proposition}
\begin{proof}
Equation~\ref{eqn:SteeredFabric} can be proven by expansion:
\begin{align}
    \nonumber
    \qdd
    &= \energize_\Lag\big[\h+\f\big] + \gamma \qd \\
    \nonumber
    &= \h+\f - \left(\frac{\qd^T\big(\M_\Lag(\h+\f)+\mxi_\Lag\big)}{\qd^T\M_\Lag\qd}\right)\qd + \frac{\qd^T\M_\Lag\f}{\qd^T\M_\Lag\qd} \qd \\
    \nonumber
    &= \h - \left(\frac{\qd^T\big(\M_\Lag\h+\mxi_\Lag\big)}{\qd^T\M_\Lag\qd}\right)\qd \\
    \label{eqn:ExpansionIntermediate}
    &\ \ \ \ + \f - \left(\frac{\qd^T\M_\Lag\f}{\qd^T\M_\Lag\qd}\right)\qd + \frac{\qd^T\M_\Lag\f}{\qd^T\M_\Lag\qd} \qd \\
    \nonumber
    &= \energize_\Lag\big[\h\big] + \f.
\end{align}

Equation~\ref{eqn:RegulatedSteeringAcrossAFabric} is just an equivalent algebraic form emphasizing the separation of acceleration along the direction of motion and orthogonal steering. It can be verfied by substituting the expressions for $\vv$ and $\wh{\vv}$. See~\cite{vanwyk2022geometricfabrics} Lemma C.2 Equation 30 for a related derivation.
\end{proof}
The system in Equation~\ref{eqn:SteeredFabric} shows that we can absorb the navigation term into the fabric thereby exposing a separate energy regulation term $\gamma\qd$. With $\gamma = 0$ the system conserves energy while $\f$ acts to steer the fabric, hence the name. With $\gamma < 0$ the system will slow to a stop. The second form given in Equation~\ref{eqn:RegulatedSteeringAcrossAFabric} shows we can also view the energy regulation as applying to the original fabric. Since the fabric is geometric and the energy regulation is an acceleration along the direction of motion, the energy increases or decreases but the path doesn't change. On top of that, the term $\mP_\perp\f$ steers the system without affecting the energy.

This section show that fabrics form a stable medium for system navigation. Navigation across the fabric can be described either as a navigation policy operating directly on the fabric accelerations or, when there is a system metric, as a force policy pushing against the fabric. Geometric fabrics, in particular, can be viewed as a road network of paths the system can travel along without any effort. We can view navigation across the fabric as a combination of energy regulation (injecting or dissipating energy) and energy invariant steering. In the case of geometric fabrics, the energy regulation doesn't change the network of paths.

The following proposition gives insight into the behavior of forced fabrics and helps guide design. It states that if we can get the system to converge with sufficient finite damping, it will converge to the zero set (goal) of the navigation policy. Strategically, we can increase the damping to slow the system as needed to give the navigation policy more influence over the behavior. And if we can prove the navigation policy converges on its own (or equivalently when forcing the Euclidean fabric), then we can construct a modified navigation policy that's guaranteed to converge to the desired goal. 

These results can guide policy design, although it's far from a complete characterization of convergence or stable policies. In many cases, we might learn a policy over a given fabric, for instance using RL. Convergence and stability are more complex in this setting, but fabrics make it easier to safely explore and find performance stable and convergent solutions.

\begin{proposition}[Convergence] \label{prop:ForcedSystemConvergence}
Let $\wt{\h}$ be a fabric and let $\f$ be a bounded navigation policy with zero set $\mathcal{S} = \{\q\:|\:\f(\q,\zero) = \zero\}$. Let $\qdd = \wt{\h}+\f = \wt{\h}_\f$ denote the forced system. Then if $\wt{\h}_\f$ converges, it converges to $\mathcal{S}$.
\end{proposition}
\begin{proof}
By Lemma~\ref{lma:FabricsAreUnbiased}, $\wt{\h}(\q,\zero) = \zero$ for all $\q$. Therefore, at convergence, $\qdd = \wt{\h}(\q^*,\zero) + \f(\q^*,\zero) = \f(\q^*,\zero) = \zero$, which implies $\q^*\in\mathcal{S}$.
\end{proof}

Proposition~\ref{prop:ForcedSystemConvergence} shows that training navigation policies can be a powerful design choice. If we enforce through structural choices the desired zero set of the navigation policy and train the policy to successfully converge, then we're guaranteed that it converges to the correct goal.

When $\f$ does not necessarily converge on its own (e.g. it may require additional damping), Theorem~\ref{thm:main} gives an explicit class of energy regulators that will guarantee convergence to $\mathcal{S}$ in the case where there exists a \emph{compatible} potential.
\section{Energy Regulation} 
\label{sec:EnergyRegulation}

This next proposition characterizes how to regulate energy within a given range $[0,\Lag_{\max}]$ using an \emph{energy regularizer} while using a navigation policy $\f$ to both modulate system energy and steer. When driven by $\f$, the system increases energy (speeds up) to a maximum energy level then maintains that energy as long as $\f$ is pushing the system forward. If the system is moving against $\f$ it removes energy (slows down). Examples of when this second case may occur are (1) the system is moving the wrong way, e.g. away from a goal; (2) the system is approaching a goal and $\f$ includes sufficient damping to bring it to rest at the goal. In both cases, the energy regularization is removed and $\f$ acts to slow the system.

\begin{proposition}[Energy Capping 1] \label{prop:EnergyCapping1}
Let $\wt{\h}$ be a fabric with Finsler energy $\Lag$ and let $\f$ be a navigation policy. Design a regularized system of the form
\begin{align}
    \qdd = \wt{\h} + \f - \lambda \M_\Lag^{-1}\B\qd,
\end{align}
where $\M_\Lag$ is the energy tensor of $\Lag$ and $\B(\q,\qd)$ is positive definite, and choose
\begin{align} \label{eqn:LambdaDesign}
    \lambda = \max\left\{
        0, \frac{\qd^T\M_\Lag\f}{\qd^T\B\qd+\gamma(\Lag)}
    \right\},
\end{align}
where $\gamma(0) = \gamma_{\max}$, $\gamma(\Lag_{\max}) = 0$, and $\Lag_{\max}$ is a desired energy cap. Then the regularized system has the following energy properties:
\begin{enumerate}
    \item Bounded energy: $\Lag\in[0,\Lag_{\max}]$.
    \item Energy increases when moving with $\f$: When $\qd^T\M_\Lag\f \geq 0$, we have $\dot{\Lag}\geq 0$ with equality only when either $\Lag=\Lag_{\max}$ or $\qd^T\M_\Lag\f = 0$.
    \item Energy decreases when moving against $\f$: When $\qd^T\M_\Lag\f < 0$, we have $\dot{\Lag} < 0$.
    \item Energy rates of change are instantaneously the same with and without the fabric, and the regularizing damper only decreases energy: $\dot{\Lag}[\wt{\h} + \f - \lambda \M_\Lag^{-1}\B\qd] = \dot{\Lag}[\f - \lambda \M_\Lag^{-1}\B\qd] \leq \dot{\Lag}[\f] = \dot{\Lag}[\wt{\h} + \f]$.
\end{enumerate}
\end{proposition}
\begin{proof}[\hspace{-23pt} Proof]
The energy derivative is
\begin{align}
    \nonumber
    \dot{\Lag} &= \qd^T\big[\M_\Lag \qdd + \mxi_\Lag\big] \\
    \nonumber
    &= \qd^T\Big[\M_\Lag \big(\wt{\h}+\f - \lambda \M_\Lag^{-1}\B\qd\big) + \mxi_\Lag\Big] \\
    \label{eqn:hDroppingOut}
    &= \qd^T\big[\M_\Lag \wt{\h} + \mxi_\Lag\big] + \qd^T\M_\Lag \big(\f - \lambda \M_\Lag^{-1}\B\qd\big)\\
    \nonumber
    &= \qd^T\M_\Lag \f - \lambda \qd^T\B\qd,
\end{align}
since $\qd^T\big[\M_\Lag \wt{\h} + \mxi_\Lag\big] = \dot{\Lag}[\wt{\h}] = 0$ by the conservation property of $\wt{\h}$.
Choosing $\lambda$ per Equation~\ref{eqn:LambdaDesign}, we have two case. If $\qd^T\M\f < 0$, then $\lambda = 0$ and
\begin{align}
    \dot{\Lag} = \qd^T\M_\Lag \f < 0.
\end{align}
This case proves property 3.

The second case is, if $\qd^T\M_\Lag\f \geq 0$, then
\begin{align}
    \lambda = \frac{\qd^T\M_\Lag\f}{\qd^T\B\qd+\gamma(\Lag)}
\end{align}
and
\begin{align}
    \dot{\Lag} &= \qd^T\M_\Lag \f - \left(\frac{\qd^T\M_\Lag\f}{\qd^T\B\qd+\gamma(\Lag)}\right) \qd^T\B\qd \\
    \label{eqn:EnergyRateOfChangeWithf}
    &= \qd^T\M_\Lag \f \left[1 - \frac{\qd^T\B\qd}{\qd^T\B\qd+\gamma(\Lag)}\right].
\end{align}
We can make two observations:
\begin{enumerate}
\item When $\gamma = 0$, $\Lag=\Lag_{\max}$ and $\qd\neq\zero$, so $\dot{\Lag}=0$.
\item When $\gamma=\gamma_{\max}$, $\Lag=0$ and $\qd=\zero$, so
    \begin{align}
        \frac{\qd^T\B\qd}{\qd^T\B\qd+\gamma(\Lag)} = 0
    \end{align}
    so $\dot{\Lag}=\qd^T\M\f \geq 0$. 
\end{enumerate}

To prove property 2, we note $\dot{\Lag} > 0$ only when $\gamma > 0$ and $\qd^T\M_\Lag\f > 0$. And $\dot{\Lag} = 0$ when either factor in Equation~\ref{eqn:EnergyRateOfChangeWithf} is zero, which means either $\qd^T\M_\Lag\f=0$ or $\qd=\zero$. The latter condition implies $\gamma=0$ and $\Lag=\Lag_{\max}$.

Property 1 follows by noting that $\dot{\Lag} = 0$ at $\Lag=\Lag_{\max}$ so $\Lag>\Lag_{\max}$ would be a contradiction.

Finally, property 4 derives from the simple observation that the contribution from $\wt{\h}$ to $\dot{\Lag}$ drops out in line~\ref{eqn:hDroppingOut} because $\wt{\h}$ is conservative. And $\lambda \geq 0$ only removes energy with its contribution being $-\lambda \qd^T\B \qd\leq 0$ since $\B$ is positive definite. 

\end{proof}

\begin{remark}
The use of $\gamma$ in the denominator of Equation~\ref{eqn:LambdaDesign} makes it robust at $\qd = \zero$. The specific profile of $\gamma$ defines how $\lambda$ moves between $\lambda_0 = \qd^T\M_\Lag\qd/\qd^T\B\qd$ (to fully cap the energy with $\dot{\Lag} = 0$ when $\Lag = \Lag_{\max}$) and $0$ when $\Lag = 0$ (equiv. $\qd = \zero$).
\end{remark}

\begin{proposition}[Energy Capping 2] \label{prop:EnergyCapping2}
Let $\wt{\h}$ be a fabric with energy $\Lag$ and let $\f$ be a navigation policy. Design a regularized system of the form
\begin{align}
    \label{eq:general_fabric}
    \qdd = \wt{\h}+ \f + \lambda \qd - \beta \qd,
\end{align}
and choose
\begin{align}
    \lambda(\alpha_f) = \gamma(\Lag) \alpha_f
\end{align}
where $\beta \in [0, \beta_{max}]$, $\gamma(\Lag) \in [0, 1]$, $\gamma(0) = 0$, $\gamma(\Lag_{max}) = 1$, and
\begin{align}
\alpha_f = -\frac{\qd^T \M_\Lag \f}{\qd^T \M_\Lag \qd}.
\end{align}
Such a system will have bounded energy $\Lag\in[0,\Lag_{\max}]$ for all time.
\end{proposition}

\begin{proof}[\hspace{-23pt} Proof]
The energy time derivative is
\begin{align}
    \dot{\Lag} &= \qd^T ( \M_\Lag \qdd + \mxi_\Lag ) \\
    &= \qd^T ( \M_\Lag (\wt{\h} + \f + \lambda \qd  - \beta \qd) + \mxi_\Lag ) \\
    &= -\beta \qd^T \M_\Lag \qd + \qd^T ( \M_\Lag (\f + \lambda \qd ))
\end{align}
since $\qd^T\big[\M_\Lag \wt{\h} + \mxi_\Lag\big] = \dot{\Lag}[\wt{\h}] = 0$ by the conservation property of $\wt{\h}$. In general, $\qd^T \M_\Lag \f$ can perform work on the system, changing its energy levels. However, system energy will ultimately be bounded given that $\gamma$ can become equal to 1 arbitrarily, and certainly $\gamma(\Lag_{max}) = 1$ by design. Whenever $\gamma = 1$, the energy time derivative becomes
\begin{align}
    \dot{\Lag}
        &= -\beta \qd^T \M_\Lag \qd + \qd^T \left( \M_\Lag \left(\f -\frac{\qd^T \M_\Lag \f}{\qd^T \M_\Lag \qd} \qd \right) \right) \\
        &= -\beta \qd^T \M_\Lag \qd
\end{align}
If $\beta = 0$, then system energy is conserved, and if $\beta > 0$, then energy is dissipated. In essence, $\gamma$ can monitor the system energy and decide how much work can be done by $\qd^T \M_\Lag \f$, which results in shifting energy levels that are ultimately bounded by $\Lag_{max}$.
\end{proof}
Within the preset boundary conditions, $\gamma$ can behave arbitrarily, fluctuating the system energy. $\gamma$ can therefore be learned from experience, enabling it to modulate system energy advantageously. In parallel, $\beta$ can also be learned, promoting dynamic braking. Note, if $\beta > 0$ and $\gamma = 1$ persists, then system energy will decrease resulting in $\|\qd\|, \|\qdd\| \to 0$ as $t \to \infty$. Note, this does not imply that $\|\f\| \to 0$ as well, but rather, the system can controllably come to rest regardless of $\|\f\|$. Finally, robustness to numerical issues when leveraging this design for $\lambda$ when $\|\qd\| \to 0$ can be obtained via the strategies in Section \ref{sec:NumericalConsiderations}.

To effectively regulate the energy of a navigation fabric to \emph{guarantee} convergence to the navigation policy's zero set, we need a measure of progress toward that zero set. That measure of progress can be given by a potential function that's compatible with the navigation policy in the sense that it's negative gradient generally points in the same direction as the policy's vector field and is (locally) minimized at the policy's zero set.

\begin{definition}[Compatible potential]
Let $\f(\q,\qd)$ be a navigation policy. We say a potential function is \emph{compatible} with $\f$ if $\partial\psi(\q) = \zero$ if and only if  $\f(\q, \zero) = \zero$ and $-\partial\psi^T\f(\q,\zero) > 0$ wherever $\f(\q,\zero) \neq \zero$ (equiv. $\partial\psi \neq \zero$).
\end{definition}

The next theorem prescribes how to regulate the energy of a navigation fabric given a compatible potential.

\begin{theorem} \label{thm:main}
Let $\energize_\Lag\big[\h(\q,\qd)\big]$ be a fabric with generator $\h$ and Finsler energy $\Lag$, and let $\f(\q,\qd)$ be a navigation policy with compatible potential $\psi(\q)$. Denote the total energy by $\Ham = \Lag + \psi$.
The system $\qdd = \mathrm{energize}_{\Ham}\big[\h + \f\big] + \gamma \qd$
with energy regulator
\begin{align}
    &\gamma(\q,\qd) = -\left(\frac{\qd\:\qd^T}{\qd^T\M_\Lag\qd}\right)\partial\psi - \beta \qd
\end{align}
converges to the zero set of $\f$ for $\beta > 0$.
\end{theorem}
\begin{proof}[Proof of Theorem~\ref{thm:main}]
The total energy $\Ham$ of the energized system $\mathrm{energize}_\Ham[\h+\f]$ is conserved by definition, and we will show that with damping it's minimized and the system comes to rest. We then show that at convergence the compatibility conditions between potential $\psi$ and perturbation field $\f$ ensure that at convergence $\f=\zero$.
 
The time derivative of the total energy $\Ham = \Ham_\Lag + \psi$ is:
\begin{align} \label{eq:HamDeriv}
    \dot{\Ham} = \qd^T\Big[\M_\Lag\qdd + \mxi_\Lag + \partial\psi\Big],
\end{align}
where $\M_\Lag \qdd + \mxi_\Lag = \zero$ are the equations of motion of $\Lag$ defined by the Euler-Lagrange equation (see \cite{vanwyk2022geometricfabrics} for a derivation). We assume $\M_\Lag$ is bounded in a finite region and strictly positive definite everywhere; in particular, it doesn't vanish or reduce rank as $\qd\rightarrow \zero$. To derive energization, we take the system
\begin{align}\label{eq:EnergizedSystem}
\qdd = \h + \f + \alpha \qd
\end{align}
and solve for the $\alpha$ which makes $\dot{\Ham} = 0$ (i.e. calculate the acceleration along the direction of motion needed to conserve energy). Plugging Eq.~\ref{eq:EnergizedSystem} into Eq.~\ref{eq:HamDeriv}, setting to zero, and solving for $\alpha$ gives:
\begin{align}
    &\qd^T\Big[\M_\Lag\big(\h + \f + \alpha \qd\big) + \mxi_\Lag + \partial\psi\Big] = \zero \\
    \nonumber
    &\Rightarrow \alpha(\qd^T\M_\Lag \qd) 
        + \qd^T(\M_\Lag\h + \mxi_\Lag) 
        + \qd^T(\M_\Lag\f+\partial\psi) = \zero \\
    \label{eq:SeparableEnergizationAlpha}
    &\Rightarrow \alpha = -\frac{\qd^T(\M_\Lag\h + \mxi_\Lag)}{Z}
        - \frac{\qd^T(\M_\Lag\f+\partial\psi)}{Z} \\
    \label{eq:EnergizationAlpha}
    &\Rightarrow \alpha = -\frac{1}{Z}\qd^T\big(\M_\Lag(\h+\f) + \mxi_\Lag + \partial\psi\big)
\end{align}
where $Z = \qd^T\M_\Lag \qd$. Equation~\ref{eq:EnergizationAlpha} is the form given in Definition~\ref{def:Energization}.

The $\alpha$ of Equation~\ref{eq:EnergizationAlpha} by definition makes the undamped equations in \ref{eq:EnergizedSystem} conserve the Hamiltonian $\Ham$, therefore the damped equations  
\begin{align} \label{eq:DampedSystem}
    \qdd = \h + \f + \alpha \qd - \beta\qd
\end{align} 
for $\beta>0$ decreases energy at a rate
\begin{align}
    \dot{\Ham}
    &= \qd^T\Big[\M_\Lag\big(\h + \f + \alpha \qd\big) + \mxi_\Lag + \partial\psi\Big]
    - \beta\qd^T\M_\Lag\qd \\ \label{eq:HamDecrease}
    &= - \beta\qd^T\M_\Lag\qd.
\end{align}
Since $\M_\Lag$ is strictly positive definite, this final expression is less than $0$ for all $\qd\neq\zero$ and 0 for $\qd = \zero$. Since $\Ham$ is always decreasing but also lower bounded, we know that its rate of decrease must converge to zero $\dot{\Ham}\rightarrow 0$ (it stops decreasing at some point). Therefore, $\dot{\Ham} = - \beta\qd^T\M_\Lag\qd \rightarrow 0$ which means $\qd \rightarrow \zero$ and hence $\qdd\rightarrow\zero$. 

Plugging $\alpha$ from Equation~\ref{eq:SeparableEnergizationAlpha} into the system in Equation~\ref{eq:DampedSystem} and taking the limit with $\qd,\qdd\rightarrow\zero$ gives
\begin{align}\nonumber
    \qdd &= \h + \f \\\nonumber
    &\ \ \ \ 
    + \left(-\frac{\qd^T(\M_\Lag\h+\mxi_\Lag)}{Z}- \frac{\qd^T(\M_\Lag\f+\partial\psi)}{Z}\right) \qd - \beta\qd \\
    &= \Big(\h - \beta \qd - \frac{\qd\qd^T}{Z}\big(\M_\Lag\h + \mxi_\Lag\big)\Big) \\ \nonumber
    &\ \ \ \ \ \ \ \ \ \ \ \ \ \ \ \ \ \ 
    + \f - \frac{\qd\qd^T}{Z}\Big(\M_\Lag\f - (-\partial\psi)\Big) \\ \label{eq:MainRes}
    &= \V + \f - \frac{\qd\qd^T}{\qd^T\M_\Lag\qd}\Big(\M_\Lag\f - (-\partial\psi)\Big).
\end{align}
Here $\V$ collects the terms in parentheses from the second line which vanish in the limit with $\V \rightarrow \zero$ as $\qd\rightarrow\zero$, and we write $-\partial\psi$ because it's the negative gradient that has positive inner product with $\f$ per the compatibility conditions. On left-hand-side we have $\qdd \rightarrow \zero$, so it's the rest of the terms in Equation~\ref{eq:MainRes} we need to analyze in the limit as $\qd\rightarrow 0$. Note that $\qd^T\qd / (\qd^T\M_\Lag\qd)$ has two factors of $\qd$ in both the numerator and the denominator. Since $\M_\Lag$ is bounded and doesn't vanish in the limit, it limits to a projection operator
\begin{align}
    \A = \frac{\vv\vv^T}{\vv^T\M_\Lag\vv},
\end{align}
where $\vv = \lim_{t\rightarrow\infty}\qd/\|\qd\|$ is the limiting direction of motion as the system comes to a stop. This notation allows us write Equation~\ref{eq:MainRes} as 
\begin{align}
    &\ \ \ \ \ \qdd = \V + \f - \frac{\qd\qd^T}{\qd^T\M_\Lag\qd}\Big(\M_\Lag\f - (-\partial\psi)\Big) \\
    \label{eq:LimitingSystemExpression}
    &\xrightarrow[t \to \infty]{} 
    \zero = \Big[\I - \A\M_\Lag\Big] \f + \A(-\partial\psi).
\end{align}
The matrix $\mP = \I - \A\M_\Lag$ has nullspace $\vv$ since
\begin{align}
    &\Big[\I - \A\M_\Lag\Big]\vv = \vv - \frac{\vv\vv^T}{\vv^T\M_\Lag \vv}\M_\Lag \vv \\
    &\ \ \ \ \ \ \ \ 
    \vv - \vv\left(\frac{\vv^T\M_\Lag \vv}{\vv^T\M_\Lag \vv}\right) = \vv - \vv = \zero.
\end{align}
Likewise, $\A$ is rank-1 with column space spanned by $\vv$, so $\Big[\I - \A\M_\Lag\Big] \f$ and $\A(-\partial\psi)$ must be linearly independent when they're both nonzero. 

We'll prove $\f = \zero$ by contradiction. $\mP\f$ and $\A(-\partial\psi)$ are orthogonal so for Equation~\ref{eq:LimitingSystemExpression} to hold, they must both be zero. If $\f\neq \zero$, then since $\mP\f = \zero$ we must have $\f\in\mathrm{span}(\vv)$. And since $\A(-\partial\psi) = \zero$, we must have either that $\partial\psi = \zero$ or $\partial\psi \perp \vv$ which implies $\partial\psi^T\f = \zero$. Both of these contradict the compatibility conditions. Therefore, $\f = \zero$.
\end{proof}

One simple way to leverage Theorem~\ref{thm:main} is to choose a potential $\psi$ whose zero set $\mathcal{S} = \{\q\:|\:\partial \psi(\q) = \zero\}$ characterizes the goal and then define $\f(\q,\qd)$ so that it's compatible with $\psi$ by construction. For instance, the following $\f$ would be compatible:
\begin{align}
    \f = -\frac{\partial\psi}{\|\partial\psi\| + \epsilon} - \B(\q,\qd) \qd.
\end{align}
The first term is the soft normalized negative gradient, and the second is a damper.
\section{Forcing energized fabrics}
\label{sec:ForcingEnergizedFabrics}

Here we analyze forcing an arbitrary fabric term $\wt{\h}(\q,\qd)$ using a forcing term pushing against a system metric of the type described in Section~\ref{sec:GeneralizedNotation} Equation~\ref{eqn:DecomposedSystem}. This a case is more specific than the general energy regulation settings discussed in Section~\ref{sec:EnergyRegulation}, but it's an important and common one used, for instance, in \cite{xie2022ngf}.
The forcing term in this case takes the form
\begin{align} \label{eq:ForcingTermWithSystemMatrix}
    \f(\q,\qd) = -\M^{-1}\partial\psi - \M^{-1}\B \qd
\end{align}
where $\M(\q,\qd)$ is an arbitrary positive definite system metric and $\B(\q,\qd)$ is an arbitrary positive semi-definite damping matrix.

$\wt{\h}$ can be an arbitrary fabric. 
For instance, we may construct a transform tree, populate its spaces with arbitrary specs, and pull them back into the root. The resulting spec $(\M,\mxi)$ defines a differential equation $\M \qdd + \mxi = \zero$ with acceleration $\qdd = -\M^{-1}\mxi = \h(\q,\qd)$. That $\h$ can then be used to generate the fabric $\wt{\h} = \energize_\Lag\big[\h\big]$ by energization. The matrix $\M$ defines the system metric which we use to define the forcing term given in Equation~\ref{eq:ForcingTermWithSystemMatrix}. If the individual specs on the transform tree are themselves geometric (the metrics are HD0 and the policies are HD2), the resulting fabric is a geometric fabric. Importantly, the metrics don't need to be Finsler (deviating from the theory of \cite{vanwyk2022geometricfabrics}), just HD0.

The following theorem shows that these systems are stable and convergent to the logical minimum of a potential function with appropriate choice of damping.

\begin{theorem}
Let $\wt{\h}(\q,\qd)$ be a fabric with positive-definite system metric $\M(\q,\qd)$, and let $\psi(\q)$ be a potential function. Then we can always find a finite positive definite damping matrix $\B(\q,\qd)$ such that the system
\begin{align}
    \qdd = \wt{\h} - \M^{-1}\big(\partial\psi + \B\qd\big)
\end{align}
converges. And at convergence, by Proposition~\ref{prop:ForcedSystemConvergence} $\psi$ is at a local minimum.
\end{theorem}
\begin{proof}
Suppose our system is
\begin{align} \label{eq:GeneralSystemForm}
    \qdd = \wt{\h} + \f + \alpha_{\Lag} \qd - \beta \qd,
\end{align}
with $\f$ as given by Equation~\ref{eq:ForcingTermWithSystemMatrix}, $\beta \in \mathbb{R}_+$, and where $\alpha_\Lag \in \mathbb{R}$ is the energization coefficient with respect to some energy $\Lag$.
Our proof follows a standard Lyapunov analysis. We design our Lyapunov function as
\begin{align}
    V = \frac{1}{2} \qd^T \M \qd + \psi.
\end{align}
The time derivative of the Lyapunov function is
\begin{align}
    \dot{V} = \qd^T \M \qdd + \frac{1}{2} \qd^T \dot{\M} \qd + \qd^T\partial \psi.
\end{align}
Plugging in $\qdd$ from Equation~\ref{eq:GeneralSystemForm} above yields
\begin{align}
    \nonumber
    \dot{V} = &\qd^T \M \Big(\wt{\h} - \M^{-1} \partial \psi - \M^{-1} \B \qd + \alpha_{\Lag} \qd - \beta \qd\Big) \\
    &\ \ \ + \frac{1}{2} \qd^T \dot{\M} \qd + \partial\psi^T\qd.
\end{align}
Rearranging and canceling terms reduces the expression to
\begin{align}
    \label{eq:Ld_simplified}
    \dot{V} = \qd^T \M \big(\wt{\h} + \alpha_{\Lag} \qd\big) + \frac{1}{2} \qd^T \dot{\M} \qd - \qd^T \B \qd - \beta \qd^T \M \qd.
\end{align}
We now write $\alpha_\Lag$ as the sum of a term $\alpha_0$ designed to remove $\wt{\h}$ and $\dot{\M}$ and a residual $\tilde{\alpha}$. I.e. $\alpha_{\Lag} = \alpha_0 + \tilde{\alpha}$ with 
\begin{align}
    \alpha_0 = \frac{-\qd^T\M\wt{\h} - \frac{1}{2}\qd^T\dot{\M}\qd}{\qd^T\M\qd}
\end{align}
so that
\begin{align}
    \qd^T \M \big(\wt{\h} + \alpha_0 \qd\big) + \frac{1}{2} \qd^T \dot{\M} \qd = 0.
\end{align}
We assume that $\Lag$, $\M$, and $\wt{\h}$ are designed such that the residual $\tilde{\alpha}$ is bounded. 
Substituting $\alpha_{\Lag} = \alpha_0 + \tilde{\alpha}$ into \ref{eq:Ld_simplified} gives
\begin{align}
    \dot{V} = \qd^T \M\big(\wt{\h} + \tilde{\alpha} \qd + \alpha_0 \qd\big) + \frac{1}{2} \qd^T \dot{\M} \qd - \qd^T \B \qd - \beta \qd^T \M \qd.
\end{align}
Regrouping yields
\begin{align}
    \dot{V} &= \Big[\qd^T \M(\wt{\h}+ \alpha_0 \qd) + \frac{1}{2} \qd^T \dot{\M} \qd \Big] \\\nonumber
    &\ \ \ \ \ \ + \tilde{\alpha} \qd^T \M \qd - \qd^T \B \qd - \beta \qd^T \M \qd.
\end{align}
The first group of terms vanish by the design of $\alpha_0$, so we get
\begin{align}
    \dot{V} = \tilde{\alpha} \qd^T \M \qd - \qd^T \B \qd - \beta \qd^T \M \qd.
\end{align}
We combine the two damping terms to produce
\begin{align}
    \dot{V} = \tilde{\alpha} \qd^T \M \qd - \qd^T \tilde{\B} \qd.
\end{align}
This equation can now be upper-bounded via the Rayleigh-Ritz theorem as
\begin{align}
    \dot{V} \leq \overline{\lambda}_M \|\qd\|^2 - \underline{\lambda}_B \|\qd\|^2,
\end{align}
where $\overline{\lambda}_M$ is the maximum eigenvalue of $\tilde{\alpha} \M$ and $\underline{\lambda}_B$ is the minimum eigenvalue of $\tilde{\B}$. Via the design of $\B$ and a sufficiently large $\beta$, we can enforce that $\underline{\lambda}_B > \overline{\lambda}_M$ yielding
\begin{align}
    \dot{V} \leq -b \|\qd\|^2,
\end{align}
where $b = \underline{\lambda}_B - \overline{\lambda}_M > 0$. We now invoke LaSalle's invariant set theorem to give $\|\qd\| \to 0$ as $t \to \infty$. This implies  $\|\qdd\| \to 0$ as $t \to \infty$, and consequently, $\|f\|, \|\partial \psi\| \to 0$ as well. This ultimately guarantees that the system will come to rest at a minimum of $\psi$.
\end{proof}

\section{Numerical Considerations}
\label{sec:NumericalConsiderations}

The mathematical definition of energization given in Definition~\ref{def:Energization} has a numerical instability at $\qd = \zero$. 
The following definition gives two robust variants that can be used for practical implementation. The choice of which to use depends on the properties of the generator being energized as discussed below.

\begin{definition} \label{def:RobustEnergization}
Let $\qdd = \h(\q,\qd)$ be an autonomous second-order differential equation, and let $\Lag$ be an energy. 
The \emph{vanishing energization} transform is defined as
\begin{align}
    &\mathrm{energize}_\Ham^\epsilon[\h] = \h + \alpha \qd \\
    &\ \ \ \ \ \ \ \ \ \ \ \mbox{with}\ \ 
    \alpha = \frac{1}{Z+\epsilon}\qd^T\big(\M_\Lag\h + \mxi_\Lag\big)
\end{align}
for $\epsilon > 0$ where $Z = \qd^T\M_\Lag\qd$. This variant smoothly reduces $\alpha$ to zero as $\qd\rightarrow \zero$ avoiding numerical instability and ambiguity at $\qd = \zero$. Another variant which we call the \emph{robust energization} transform additionally preserves the unbiased property of energization while resolving numerical issues:
\begin{align}
    \mathrm{energize}_{\Ham}^{\eta_\sigma,\epsilon}\big[ \h \big] =
        \eta_\sigma(\|\qd\|)\; \mathrm{energize}_{\Ham}^{\epsilon}\big[ \h \big]
\end{align}
where $\eta_\sigma(s)$ is some function that diminishes to zero as $s\rightarrow 0$ with length scale $\sigma$. For instance, $\eta_\epsilon(s) = 1 - \exp\big\{-s^2/(2\sigma^2)\big\}$ is a common choice.
\end{definition}
The vanishing energization transform is the same as the standard energization transform aside from the $\epsilon$ in the denominator. When the generator $\h$ is unbiased (zero at $\qd = \zero$), this transformed system is also unbiased. The robust energization transform is useful when energizing a biased generator to create an unbiased system. It explicitly includes the $\eta_\sigma$ term to ensure the resulting system is zero at $\qd = \zero$ (unbiased).
\section{Conclusions}
\label{sec:Conclusions}

This paper reformulates fabrics to focus on their fundamental stability as a medium for policies to operate across. The fabric creates a nominal prior behavior which guides the policy. The policy then steers across the system and regulates energy. When the fabric is geometric, it forms a well-defined road network of paths that the system wants to follow. This reformulation is more intuitive than previous formulations, while subsuming those formulations, making the fabrics both flexible and easier to use in practice, particularly for learning applications.

\bibliographystyle{plainnat}
\bibliography{references}
\clearpage

\end{document}